\documentclass[11pt]{article}

\usepackage{amsmath}
\usepackage{amsfonts}
\usepackage{amsthm}
\usepackage{graphicx}

\newtheorem{theorem}{Theorem}[section]
\newtheorem{corollary}[theorem]{Corollary}
\newtheorem{lemma}[theorem]{Lemma}
\newtheorem{proposition}[theorem]{Proposition}
\newcommand{\BlackBox}{\rule{1.5ex}{1.5ex}}
\renewenvironment{proof}{\par\noindent{\bf Proof.\ }}{\hfill\BlackBox\\[2mm]}

\long\def\acks#1{\vskip 0.3in\noindent{\large\bf Acknowledgments}\vskip 0.2in \noindent #1}

\allowdisplaybreaks

\def\B{\mathcal{B}}
\def\cX{\mathcal X}
\def\RR{\mathbb R}
\def\bE{\mathbf E}
\def\bx{\mathbf x}
\def\by{\mathbf y}
\def\bc{\mathbf c}
\def\bK{\mathbf K}

\def\L2p{{L^2_{\rho_{\!_\cX}}}}
\def\H{\mathcal H}
\def\calX{\mathcal X}
\def\calH{\mathcal H}

\def\calO{\mathcal O}
\def\dint{\displaystyle\int}

\def\mse{{\rm mse}}

\let\citep\cite

\begin{document}

\title{\bf Learning Theory of Distributed Regression with Bias Corrected Regularization Kernel Network}

\author{{\bf Zheng-Chu Guo}\\
\small School of Mathematical Sciences, Zhejiang University, \\
\small Hangzhou 310027, P. R. China\\
\small Email: guozhengchu@zju.edu.cn \\
 \\
{\bf Lei Shi}\\
\small Shanghai Key Laboratory for Contemporary Applied Mathematics, \\
\small School of Mathematical Sciences, Fudan University, \\
\small Shanghai 200433, P. R. China\\
\small Email: leishi@fudan.edu.cn \\
\\
{\bf Qiang Wu} \\
\small Department of Mathematical Sciences, Middle Tennessee State University, \\
\small Murfreesboro, TN 37132, USA\\
\small Email: qwu@mtsu.edu
}
\date{}

\maketitle

\begin{abstract}%
Distributed learning is an effective way to analyze big data. In distributed regression, a typical approach is to
divide the big data into multiple blocks, apply a base regression algorithm on each of them,
and then simply average the output functions learnt from these blocks.
Since the average process will decrease the variance, not the bias, bias correction
is expected to improve the learning performance if the base regression algorithm is a biased one.
Regularization kernel network is an effective and widely used method for nonlinear regression analysis.
In this paper we will investigate a bias corrected version of regularization kernel network.
We derive the error bounds  when it is applied to a single data set and
 when it is applied as a base algorithm in distributed regression.
We show that, under certain appropriate conditions,
 the  optimal learning rates can be reached in both situations.

\end{abstract}

\section{Introduction}

Data acquisition become much fast and easier as the development of technology.
In this big data era, distributed learning has received considerable attention and
is shown to be an effective way to analyze data that is so big and cannot be
handled by a single machine.  Among various distributed learning paradigms,
a  simple one is to divide the whole data set into multiple blocks, apply  a base learning algorithm
to each block, and then average the results from different blocks
\citep{rosenblatt2016optimality, zhang2015divide}. This process, though simple,
has some advantages. First, it is computational efficient because the second stage can be easily parallelized.
Second, because no mutual communication is required, the data security or confidentiality can be well protected.
Last, recent research shows this method is consistent and sometimes reaches optimal learning rate
\citep{zhang2015divide, LGZ16}.
Thus its asymptotic effectiveness is theoretically guaranteed.

In distributed learning the performance highly depends on the selection of the base algorithm in the second stage.
Assume a big data set $D$ of $N$ observations is randomly divided into $m$ blocks, $D_1, D_2, \ldots, D_m,$
which are assumed to be of the same size at the moment so that $D_i$ are independent and identically distributed
if the entire sample set $D$ is independently drawn from some unknown distribution $\rho$.
Let $\hat f_1, \hat f_2, \ldots, \hat f_m$ be the estimators obtained
by applying a base algorithm on these data blocks.
Assume each estimator $\hat f_i$ has bias $b$ and variance $v$ .
Then the mean squared error of $\hat f_i$ is
$\mse(\hat f_i) = b^2 + v$ while the average estimator
$$\bar f = \frac 1 m \sum_{i=1}^ m \hat f_i$$ has
$\mse(\bar f) = b^2 + \frac {v}{m}.$
On a single data block the algorithm usually trades off the bias and variance well to achieve the optimal performance.
In distributed learning, however, the variance shrinks fast when $m$ is large but the bias keeps unchanging during
the average process.  In this case, the bias may dominate the learning performance.
An algorithm (or a model selection strategy) that is optimal for a single block is
not necessarily still optimal for distributed learning.
Instead, distributed learning prefers algorithms of small bias as the base learning algorithm on each block.
Therefore, when a base learning algorithm is biased,
bias correction is expected to play a role to improve the performance.
The purpose of this paper is to investigate the application of biased corrected
regularization kernel network for distributed regression analysis.

In regression analysis, the data $D = \{ (x_1, y_1), (x_2, y_2), \ldots, (x_{|D|}, y_{|D|})\}$
is a set of observations collected for input variable $X$ of predictors and a scalar response
variable $Y,$ where $|D|$ is the sample size of the data set $D.$ Assume they are  linked by
$$ y_i = f^*(x_i)+\epsilon_i,\qquad i=1,2,\ldots, |D|,$$
where $x_i$ comes from a compact metric space (e.g., a bounded subset in $\RR^p$),
$y_i\in\RR,$ and $\epsilon_i$  is a zero-mean noise.
The target is to recover the unknown true model $f^*$ as accurate as possible
to understand the impact of predictors and predict the response on unobserved data.
Numerous regression methods have been developed in the literature, e.g. ridge regression,
LASSO, and regularization kernel network (RKN). Among them,
the regularization kernel network  is a popular kernel method for nonlinear regression analysis.
Its predictive consistency has been extensively studied in a vast literature;
see e.g. \citep{EPP, bousquet2002stability, zhang2003leave, devito2005model, wu2006learning, bauer2007regularization,
caponnetto2007optimal, smale2007learning,  sun2009note, steinwart2009optimal, GXGZ2017}
and the references therein.
Its applications were also extensively explored and shown successful in  many problem domains.
More recently, a bias corrected version for RKN, or BCRKN for short, was proposed in \citep{wu2017bias}
to improve the performance of block wise data processing. In \citep{wu2017bias},
the asymptotic bias and variance of BCRKN on a single data set was characterized,
which indicates BCRKN has smaller bias than RKN and thus implies its efficiency in
learning with block wise data intuitively. Empirical study also confirmed this.
However, without rigorous analysis of the error bounds, there is lack of theoretical guarantee.
In this paper, we will derive the error bounds and learning rates of BCKRN both for a single data set
and for distributed regression. This will provide a theoretical guarantee for the use of
BCRKN from a learning theory perspective.

The rest of this paper will be arranged as follows. In Section \ref{sec:results} we will describe the BCRKN algorithm
and state our main results. In particular, we show that BCRKN can achieve the minimax optimal rates in both
single data learning and distributed learning. Moreover, BCRKN relaxes the saturation effect of RKN.
In Section \ref{sec:discuss} we discuss the relations of our results with existing work and conduct some
comparisons. The proof of our results are given in Sections \ref{sec:pre}-\ref{sec:distributed}.

\section{Main results}
\label{sec:results}

Let $\calX$ denote the input space which is assumed to be a compact metric space. A Mercer kernel on $\calX$ is
 a continuous, symmetric, and positive-semidefinite function $K: \calX \times \cal X\to \RR$.
 The function class spanned by $\{K_x = K(x, \cdot): x\in\calX\}$  and equipped with the inner product
 satisfying $\langle K_x, K_t\rangle_K =K(x, t)$ forms a pre-Hilbert space. Its completion
 is called a reproducing kernel Hilbert space (RKHS) $\H_K$ associated to the kernel $K$,
 with the name coming after the  reproducing property $f(x) = \langle f, K(x, \cdot)\rangle_K,$ $\forall f\in\H_K.$
 Note that  $|f(x)|\le \sqrt{K(x,x)} \|f\|_K$ for all $f\in\H_K$. Consequently,
 with $\kappa = \sup_{x\in\calX}\sqrt{K(x,x)}<\infty,$
  $\calH_K$ can be embedded into $C(\calX)$ and $\|f\|_\infty \le \kappa\|f\|_K.$
More other properties of RKHS that will not be used in this paper can be found in  \citep{Aron}.

Given the data $D$ and the RKHS $\H_K,$  RKN estimates the true model by
 \begin{equation} \label{eq:rkn}
  f_{D, \lambda} = \arg\min_{f\in\H_K}  \frac 1 {|D|} \sum_{i=1}^{|D|} (y_i-f(x_i))^2 +\lambda \|f\|_K^2,
 \end{equation}
 where $\lambda >0$ is a regularization parameter that trades off the fitting error and model complexity.
The well known representer  theorem \citep{wahba1990spline} tells that
$$f_{D, \lambda} (x) = \sum_{i=1}^{|D|} c_i K(x_i, x)$$ with the coefficients
$\bc = (c_1, \ldots, c_{|D|})^\top$ satisfying  $(\lambda |D| I + \bK)\bc = \by_D$
where $\bK=(K(x_i,x_j))_{i,j=1}^{|D|}$ is the  kernel matrix on the input data $\bx_D=\{x_1, \ldots, x_{|D|}\}$ and
$\by_D =(y_1, \cdots, y_{|D|})^\top$ is the vector of the response data.
Let $S_D:\H_K\to \RR^{|D|}$ be the sampling operator
defined by $$S_D f=(f(x_1), \ldots, f(x_{|D|}))^\top, \qquad \forall\ f\in\H_K.$$
Its dual operator $S_D^*$ is given by
$$S_D^*\bc = \sum_{i=1}^{|D|} c_iK_{x_i} \in \H_K, \qquad \forall \ \bc\in\RR^{|D|}.$$
Then  $f_{D, \lambda}$ has the following operator representation \citep{smale2007learning}
\begin{equation}
\label{eq:rknop}
f_{D, \lambda} = \tfrac 1 {|D|} \left(\lambda I + \tfrac 1 {|D} S_D^*S_D\right)^{-1}S_D^* \by_D .
\end{equation}
Note that the operator $\frac 1 {|D|} S_D^*S_D$ is a sample version of the integral operator
$$L_Kf(x) = \bE_t \left[K(x, t) f(t)\right]= \int_{\calX} K(x, t) f(t) \hbox{d}\rho_\calX(t)$$
where $\rho_\calX$ is the marginal distribution of $\rho$ on $\calX.$
Recall that $L_K$ defines a compact, symmetric, and positive operator on $\H_K.$
In the sequel we also use the notation  $L_{K, D} = \frac{1}{|D|} S_D^*S_D$ and  write
$$f_{D,\lambda}=(\lambda I+L_{K, D})^{-1}\left( \tfrac1{|D|} S_D^* \by_D\right),$$

By the aid of operator representation \eqref{eq:rknop}, the asymptotic bias of RKN
can be characterized as $-\lambda (\lambda I +L_K)^{-1} f^*.$
The bias corrected regularization kernel network (BCRKN) is defined by
subtracting an plug-in estimator of the bias \citep{wu2017bias}
 \begin{equation} \label{eq:bcrkn}
 f^\sharp_{D, \lambda} = f_{D,\lambda} + \lambda \left(\lambda I + L_{K, D}\right)^{-1}  f_{D, \lambda}.
\end{equation}
It is also verified in \citep{wu2017bias} that
 $$f^\sharp_{D, \lambda}(x) = \sum_{i=1}^n c_i^\sharp K(x_i, x)$$ with
 $\bc^\sharp = \bc + \lambda\left(\lambda I + \frac 1 n\bK\right)^{-1} \bc.$
The effectiveness of BCRKN has been tested empirically by a variety of
simulations and real applications in \citep{wu2017bias}. The main purpose
of this paper is to verify its effectiveness in distributed regression
from a learning theory perspective.

To perform rigorous error analysis and present our main results,
we need some notations and assumptions that are used throughout the paper.
Note that we can extend the domain of $L_K$ to $\L2p$ and obtain
a compact, symmetric, and positive operator on $\L2p,$
which will be denoted by $L$.  We can in turn say $L_K$ is the restriction of $L$ on $\H_K.$
So $Lf=L_Kf$ for $f\in\H_K$ and we do not need to differentiate them when
operating on functions in $\H_K.$
Our first assumption is a regularity condition on the true model:
\begin{equation}\label{regularitycondition}
         f^*=L^r (u^*) ~~{\rm for~some}~r>0~{\rm and} ~ u^*\in \L2p.
\end{equation}
This assumption has been widely used in the literature of learning theory
to characterize the approximation ability of $\H_K;$ see e.g.
\citep{devito2005model, smale2007learning, bauer2007regularization,
zhang2015divide} and many references therein.
Recall that $L^\frac 12$ is an isomorphism from $\overline{\H_K}$
onto $\H_K$, i.e.
\begin{equation}\label{eq:iso}
\|f\|_\L2p = \|L^\frac 12 f\|_K, \qquad \hbox{ for } f\in\overline{\H_K},
\end{equation}
where $\overline{\H_K}$ is the closure of $\H_K$ in $\L2p.$
So if $r\ge \frac 12$, the condition \eqref{regularitycondition} implies $f^*\in \H_K.$

We shall use the  effective dimension
$\mathcal{N} (\lambda)={\rm Tr}((L_K + \lambda I)^{-1}L_K),$
that is, the trace of $(L_K+ \lambda I)^{-1}L_K,$
to measure the complexity of $\mathcal H_K$ with respect to $\rho_{\!_\cX}.$
We assume that there exist a
constant $C_0>0$ and some $0<\beta\le 1$ such that for all $\lambda>0$
\begin{equation}\label{effecdim}
     \mathcal N(\lambda)\leq C_0\lambda^{-\beta}.
\end{equation}
Again this is a natural and widely used assumption in the literature; see e.g.
\citep{devito2005model, caponnetto2007optimal, zhang2013divide, LGZ16}.

Assume $\kappa \ge 1$ without loss of generality for otherwise we can define $\kappa = \max\{1, \sup\limits_{x\in\calX} \sqrt{K(x,x)}\}.$
Denote
\begin{equation*}
           \mathcal B_{|D|,\lambda}=\frac{2\kappa}{\sqrt{|D|}}\left\{\frac{\kappa}{\sqrt{|D|\lambda}}
           +\sqrt{\mathcal{N}(\lambda)}\right\},
\end{equation*}
where $|D|$ is the sample size of the data set $D.$

The consistency of RKN as well as  BCRKN generally requires the regularization parameter $\lambda$ to be chosen
according to the sample size and satisfies $\lambda \to 0$ and $\lambda |D| \to \infty$ as  $|D|\to \infty.$
This implies $\lambda$ is upper bounded by an absolute constant.
So, in the sequel, we will assume $\lambda \le 1$  without loss of generality
  to simplify our notations and presentations.

As the performance of distributed learning highly depends on the base algorithm,
we will conduct a thorough error analysis of BCRKN for a single data set first
and then turn to the distributed regression.

\subsection{Error bound for learning with a single data set}

We derive the following error bounds and learning rates for BCRKN when it is applied on a single data set.

\begin{theorem}\label{thm:error2}
If the regularity condition
\eqref{regularitycondition} holds with $0< r\le 2$ and $0<\lambda \le 1$, then for
any $0<\delta< 1 ,$ with confidence at least $1-\delta$,
\begin{equation}\label{eq:boundP}
\|f_{D,\lambda}^\sharp-f^*\|_\L2p\le
          C \left(\frac{\mathcal B_{|D|,\lambda}} {\sqrt{\lambda}}
          +1\right)^{3} \left(\mathcal B_{|D|,\lambda}+\lambda^r\right)
          \left(\log\frac{4}{\delta}\right)^4,
\end{equation}
where  $C$ is a constant independent of $|D|$ or $\delta$. Consequently, we have
\begin{equation}\label{eq:boundE}
\bE\left[\|f_{D,\lambda}^\sharp-f^*\|^2_\L2p\right]
\le  4\Gamma(9)C^2 \left(\frac{\mathcal B_{|D|,\lambda}} {\sqrt{\lambda}}
          +1\right)^{6} \left(\mathcal B_{|D|,\lambda}+\lambda^r\right)^2.
\end{equation}
\end{theorem}

\begin{corollary}\label{rate}
Assume  the regularity condition
\eqref{regularitycondition} holds with $0< r\le 2$
and
\eqref{effecdim} holds with $0< \beta \leq 1.$
\begin{enumerate}
\item[(i)] If $0<r<\frac 12$, choose  $\lambda=|D|^{-\frac{1}{1+\beta}}.$ Then
for any $0< \delta<1$, with confidence at least $1-\delta,$ we have
\begin{equation*}
        \|f_{D,\lambda}^\sharp-f^*\|_\L2p \leq {C}_1 |D|^{-\frac{r}{1+\beta}} \left(\log\frac{4}{\delta}\right)^4,
\end{equation*}
where  ${C}_1$ is a constant independent of $|D|$ or $\delta$.
Consequently,
\begin{equation*}
        \bE\left[\|f_{D,\lambda}^\sharp-f^*\|^2_\L2p\right] =\mathcal O \left(|D|^{-\frac{2r}{1+\beta}}\right).
\end{equation*}

\item[(ii)] If $\frac12 \le r\le2$, choose  $\lambda=|D|^{-\frac{1}{2r+\beta}}$. Then for any $0< \delta<1$, with confidence at least
$1-\delta,$ we have
\begin{equation*}
        \|f_{D,\lambda}^\sharp-f^*\|_\L2p \leq {C}_2 |D|^{- \frac{ r}{2r+\beta}}\left(\log\frac{4}{\delta}\right)^4,
\end{equation*}
 where  $\tilde{C}_2$ is a constant independent of $|D|$ or $\delta$.
Consequently,
\begin{equation*}
        \bE\left[\|f_{D,\lambda}^\sharp-f^*\|^2_\L2p\right] =\mathcal O \left(|D|^{-\frac{2r}{2r+\beta}}\right).
\end{equation*}
       \end{enumerate}
\end{corollary}

Recall that the minimax optimal learning rate under the assumptions \eqref{regularitycondition} and
\eqref{effecdim} is $\mathcal O \left(|D|^{-\frac{ 2r}{2r+\beta}}\right).$
Theorem \ref{thm:error2} tells that, when $r\ge \frac 12$,  BCRKN achieves the minimax optimal learning rate on a single data set.

Since $f^*\in \H_K$ when $r\ge \frac 12$, we can also measure the convergence of $f_{D, \lambda}$ to $f^*$ in $\H_K$.
As pointed out in \citep{smale2007learning},
the convergence in $\H_K$ implies the convergence in $C^s(\cal{X})$
 if $K\in C^{2s}({\cal X}\times {\cal X}),$ here $C^s(\cal{X})$ is the space of all functions on ${\cal X}\subset\RR^p$
 whose partial derivatives up to order $s$ are continuous with $\|f\|_{C^s(\cal{X})}=\sum_{|\alpha|\le s}\|D^\alpha f\|_\infty.$ So the convergence in $\H_K$ is much stronger. It is not only for the target function itself,
 but also for its derivatives.

\begin{theorem}\label{thm:errorK}
If the regularity condition \eqref{regularitycondition} holds with $\frac12< r\le 2$,
then for any $0< \delta<1$ with confidence at least $1-\delta$,
\begin{equation}\label{boundK}
 \|f_{D,\lambda}^\sharp-f^*\|_{K}\le
          C_K  \left(\frac{\mathcal B_{|D|,\lambda}} {\sqrt{\lambda}}   +1\right)^2
          (\lambda^{-\frac12}\mathcal B_{|D|,\lambda}+\lambda^{r-\frac12})\left(\log\frac{4}{\delta}\right)^3,
\end{equation}
where  $C_K$ is a constant independent of $|D|$ or $\delta$. If
\eqref{effecdim} holds with $0< \beta \leq 1$ and
$\lambda=|D|^{-\frac{1}{2r+\beta}}$, then for any $0< \delta<1$, with confidence at least
$1-\delta,$ we have
\begin{equation}\label{rateK}
        \|f_{D,\lambda}^\sharp-f^*\|_K \leq \tilde{C}_K |D|^{-\frac{ r-\frac12}{2r+\beta}}\left(\log\frac{4}{\delta}\right)^3,
\end{equation}
where  $\tilde{C}_K$ is a constant independent of $|D|$ or $\delta$. Moreover,
\begin{equation}\label{rateKE}
        \bE\left[\|f_{D,\lambda}^\sharp-f^*\|^2_{K}\right] =\mathcal O \left(|D|^{-\frac{2r-1}{2r+\beta}}\right).
\end{equation}
\end{theorem}

Under the assumptions \eqref{regularitycondition} with $r>\frac12$  and
\eqref{effecdim} with $0<\beta\le1,$ the minimax optimality of the bound
$\mathcal{O}\left(|D|^{-\frac{2r-1}{2r+\beta}}\right)$ in the $\H_K$-metric has been proved in \citep{GFZ2016}.
Theorem \ref{thm:errorK} indicates that the stronger convergence of BCRKN
is also rate optimal in the minimax sense.

When $0<r < \frac 12,$ we are unfortunately not able to obtain the minimax rate
by the integral operator technique under
the assumption \eqref{effecdim}. Note that if $\H_K$ is finite dimensional
the range of $L^r$ is exactly $\H_K$ for all $r>0.$
The assumption \eqref{regularitycondition} always implies $f^*\in\H_K$.
So the situation $0<r<\frac 12$
makes sense only when $\H_K$ is infinite dimensional.
In this case, $L_K$ has infinite positive eigenvalues which converge to $0$.
This imposes the main difficulty of error analysis via integral technique ---
 although $L_{K, D}$ converges well to $L_K$
at a rate $\mathcal O(|D|^{-1/2})$, the difference of $(\lambda I + L_{K, D})^{-1}$ and $(\lambda I + L_K)^{-1}$
cannot be well bounded when $\lambda \to 0.$ Actually, even for RKN which has been exhaustedly studied in the literature,
it is an open problem to obtain the minimax rate under the assumptions \eqref{regularitycondition} and \eqref{effecdim}.
However, if there is sufficient amount of unlabeled data which helps to improve the estimate of the integral operator,
minimax rate can be achieved. For this purpose we propose the following semi-supervised approach.
Assume, in addition to the labeled data $D$, we have sequence of unlabelled data
$x_{|D|+1}, \ldots, x_{|D'|}.$ We create a fully labeled data set
$${D'}=\{(x_1,{y}_1'),\cdots,(x_{|D|},{y}_{|D|}'),(x_{|D|+1},0), \cdots, (x_{|D'|},0)\},$$
where  ${y'_i}=\frac{|D'|}{|D|}y_i$ for $1\le i\le |D|.$
Then we can apply RKN and BCRKN on $D'$ to obtain semi-supervised estimators $f_{D',\lambda}$
and $f_{D', \lambda}^\sharp.$
Note that $D=D'$ when $|D'|=|D|.$ So the unsupervised methods can be regarded as extensions of supervised methods while
the supervised methods is a special case of semi-supervised methods with no unlabeled data.
The next theorem confirms that BCRKN can achieve the minimax rate for $0<r<\frac12$
when there are enough unlabeled data.

\begin{theorem}\label{thm:error1}
Assume the regularity condition \eqref{regularitycondition} with $0< r<\frac12$.
For any $0<\delta<1,$ we have with confidence at least
$1-\delta$,
\begin{equation}\label{eq:semibound}
\|f_{D',\lambda}^\sharp-f^*\|_\L2p\le \left(\frac{2M }{\kappa}+ 4\|u^*\|_\L2p\right)\left(\frac{\B_{|D'|,\lambda}}
           {\sqrt{\lambda}}   +1\right)^3\left( \B_{|{D}|,\lambda}+ \lambda^r \right) \left(\log\frac{4}{\delta}\right)^3.
\end{equation}
If in addition
\eqref{effecdim} holds with $0< \beta \leq 1$ and $r+\beta\ge \frac12,$
$\lambda=|D|^{-\frac{1}{2r+\beta}}$, $|D'|\ge |D|^{\frac{1+\beta}{2r+\beta}}.$
For any $\delta\in(0,1),$ with confidence at least $1-\delta,$ there holds
\begin{equation}
\label{eq:semirate}
\|f_{D',\lambda}^\sharp-f^*\|_\L2p\le C'|D|^{-\frac{r}{2r+\beta}}\left(\log\frac{4}{\delta}\right)^3.
\end{equation}
where the constant $C'$ is independent of $\delta$, $|D|$ or $|D'|$ and will be given explicitly in the proof.
\end{theorem}

\subsection{Error bound of distributed regression with BCRKN}
\label{sec:dr-result}

When BCRKN is used as a base algorithm for distributed regression,
a big data set  $D$ is  split into $m$ blocks $D_1, D_2, \ldots, D_m,$
on each block $D_j$, BCRKN is applied to produce an estimator $f_{D_j, \lambda}^\sharp,$
and the weighted average of $f_{D_j, \lambda}^\sharp,$
\begin{equation}\label{distributedlearningalgorithm}
\overline{f}_{D,\lambda}^{\sharp}=\sum_{j=1}^m \frac{|D_j|}{|D|} f_{D_j, \lambda}^{\sharp}.
\end{equation}
is used for the purposes of prediction and inference.
For this divide-and-conquer approach, we first give a general error bound for an arbitrary $m$.
Here we do not require each block has the same sample size.

\begin{theorem}\label{thm: distributed main result1}
If the regularity condition
\eqref{regularitycondition} holds with $\frac12\le r\le 2$ and $\lambda \le 1$, then
there exists a constant $\bar{C}$ independent of $m$ or $|D_j|$
such that
\begin{equation*}
\bE\left[\|\overline{f}_{D,\lambda}^\sharp-f^*\|_\L2p^2\right] 
\le \bar{C}
      \sum_{j=1}^m\frac{|D_j|}{|D|}\left(\frac{\mathcal B_{|D_j|,\lambda}}{\sqrt{\lambda}} +1\right)^6
      \left(\frac{|D_j|}{|D|}\mathcal B^2_{|D_j|,\lambda} 
      +\frac{\lambda^2\mathcal{N}(\lambda)}{|D|}+\lambda^{2r}\right).
\end{equation*}
\end{theorem}

We next show that the distributed BCRKN
\eqref{distributedlearningalgorithm} can achieve the optimal learning rates.
provided that $m$ is not too large.

\begin{theorem}\label{thm: distributed main result2}
Assume the regularity condition
\eqref{regularitycondition} with $\frac12\le r\le 2$. If
\eqref{effecdim} holds with $0< \beta \leq 1$,
$|D_1|=|D_2|=\dots=|D_m|$, $\lambda=|D|^{-\frac{1}{2r+\beta}}$, and the number of the local machines satisfies
\begin{equation}\label{eq:mBCRKN}
         m\leq |D|^{\min\left\{\frac2{2r+\beta},\frac{2r-1}{2r+\beta}\right\}},
\end{equation}
then
\begin{equation*}
         \bE\left[\|\overline{f}_{D,\lambda}^\sharp-f^*\|_\L2p^2\right]
         =\mathcal O\left(|D|^{-\frac{2r}{2r+\beta}}\right).
\end{equation*}
\end{theorem}

\section{Relations to existing work and discussions}
\label{sec:discuss}

The minimax analysis of regularized least square algorithm has received attention
in statistics and learning theory literature; see e.g. \citep{devore2004mathematical,
gyorfi2006distribution, temlyakov2008approximation, caponnetto2007optimal,
steinwart2009optimal}. In particular,
assume $L_K$ admits  an eigendecomposition $L_K=\sum_{i=1}^\infty
\tau_i \phi_i\otimes \phi_i,$ where $\tau_i\ge 0$ and $\phi_i$ are the eigenvalues and
 eigenfunctions of $L_K$, respectively.
It is proved in \citep{caponnetto2007optimal} that, if the regularity condition \eqref{regularitycondition} holds
with some $r\ge \frac 12$ and
the eigenvalues satisfy $\tau_i\sim i^{-2\alpha}$
for some $\alpha>\frac 12$, then the minimax optimal learning rate of regularized least square algorithm
is $\calO(|D|^{-\frac {2\alpha}{4\alpha r +1}}).$
It is also proved that RKN can achieve minimax rate if $\frac 12 <r \le 1.$
 When $r=\frac 12$,  they obtained a suboptimal rate
$\calO\left(\left(\frac{\log |D|} {|D|}\right)^{-\frac {2\alpha}{2\alpha+1}}\right).$
In \citep{steinwart2009optimal}, under the additional restriction
$$ \|f\|_\infty \le C \|f\|_K^{\frac 1 {2\alpha}} \|f\|_\L2p^{1-\frac 1 {2\alpha}}, \qquad \forall \ f\in\H_K$$
it is proved that the projected (or clipped) RKN estimator can achieve the minimax learning rate.
More recently, in \citep{LGZ16} it is proved that RKN can achieve minimax learning rate
for $r$ in the whole range of $[\frac 12, 1]$ without any restrictions except for the conditions \eqref{regularitycondition}
and $\tau_i\sim i^{-2\alpha}$ and thus improves the results in \citep{caponnetto2007optimal, steinwart2009optimal}.
When $r\ge 1,$ RKN suffers the saturation effect and the learning rate will not improve.
Note our condition \eqref{effecdim} on the effective dimension is nearly equivalent to $\tau_i\sim i^{-2\alpha}$
with $\beta = \frac 1 {2\alpha}$. The result in Corollary \ref{rate} tells that
BCRKN can achieve the minimax learning rate for $r\in [\frac 12, 2]$ and thus relaxes the
saturation effect of RKN.

For distributed regression problem, assume all data blocks $D_i,\ i=1,\ldots, m$, are of equal size.
If  RKN is used as the base algorithm, under the assumptions that $\bE[|\phi_i(x)|^{2k}] \le A^{2k}$
for some $k>2$ and constant $A<\infty$, $\lambda_i \le a i^{-2\alpha}$, and $f^* \in \H_K$ (i.e. $r=\frac 12$), it is proved
in \citep{zhang2015divide} that the optimal learning rate of $\calO(n^{- \frac{2\alpha}{2\alpha+1}})$
can be achieved by choosing  $\lambda = |D|^{-\frac{2\alpha}{2\alpha+1}}$  and
restricting the number of local processors
$$m \le c_\alpha \left(\frac{|D|^{\frac{2(k-4)\alpha-k}{2\alpha+1}}}{A^{4k}\log^k |D|}\right)^{\frac 1 {k-2}}.$$
Later in \citep{LGZ16} the regularity condition \eqref{regularitycondition} was taken into consideration
and it is proved that the distributed regression can achieve the minimax optimal rate for
all $r\in[\frac 12, 1]$ if
\begin{equation}\label{eq:mRKN}
m \le |D|^{\min\{\frac{6\alpha(2r-1)+1}{5(4\alpha r +1)}, \frac{2\alpha(2r-1)}{4\alpha r +1}\}}.
\end{equation}
The method suffers from the saturation effect inherited from RKN. So the learning rate cannot improve
with $r>1.$ When BCRKN is applied as the base algorithm for distributed regression,
the saturation effect is relaxed and the minimax optimal learning rate can be achieved
for the whole range $r\in [\frac 12, 2]$ as in the single data learning case.
Compare \eqref{eq:mBCRKN} with \eqref{eq:mRKN} and
we see our analysis also relaxes the restriction on the number  $m$ of local processors.

When $r<1$ we notice that distributed regression with RKN and BCRKN both
reach the optimal rates by underregularization, that is, selecting the
regularization parameter according to the number of all observations $|D|$,
not the number of observations in each block $|D_i|.$
But due to the reduced bias the parameter selection of BCRKN is less sensitive
and thus could be advantageous in practice.
We show this by an illustrative example used in \citep{zhang2015divide}.
Consider the model $f^*(x) = \min\{x, 1-x\}$ with $x\,{\small\sim}\,\hbox{Uniform}[0, 1]$
and the noise $\epsilon\,{\sim}\,N(0, \sigma^2)$  with $\sigma^2 =\frac 15.$
Let $K(x, t) = 1+\min\{x, t\}$. Then $f^*\in \H_K$ and $\|f^*\|_K=1.$
We first compare the distributed RKN and the distributed BCRKN
when $\lambda = |D|^{-2/3}$, a theoretically optimal choice.
We generate $|D| = 4098$ sample points and use number of partitions
$m\in\{2, 4, 8, 16, 32, 64, 128, 256, 512, 1024\}.$
The mean squared errors of two methods are plotted in Figure \ref{figzh} (a).
We see BCRKN slightly outperforms RKN for all $m$.

Recall that the analyses in \citep{zhang2015divide, LGZ16} and this paper
indicate the optimal choice of the regularization parameter is $\lambda = |D|^{-\theta}$
with $\theta$ an index depending on the regularity of the true target function $f^*$ and
the effective dimension of the integral operator $L_K$. Clearly both are unknown in practice
and thus a theoretical optimal choice of the regularization parameter is actually not available.
At the same time, in a big data setting where distributed regression is necessary
globally tuning the optimal parameter is either impossible or too time consuming.
A reasonable way is to tune the parameter locally to get optimal choice $\lambda_i = |D_i|^{-\theta}$
on $D_i$ and then underregularize it by using $\lambda = \lambda_i ^{\frac{\log |D|}{\log |D_i|}}=|D|^{-\theta}.$
So we next compare the use of RKN and BCRKN in distributed regression when
this parameter selection strategy is used.
The results are shown in Figure \ref{figzh} (b).
We see the requirement on the number of local processors becomes more restrictive for both methods,
indicating that underregularing locally optimal parameter does not lead to globally optimal parameter.
BCRKN significantly outperforms RKN as $m$ increases, indicating it is less sensitive to the parameter selection
when a globally optimal parameter is not available.

\begin{figure}[h]
\includegraphics[width=0.49\textwidth]{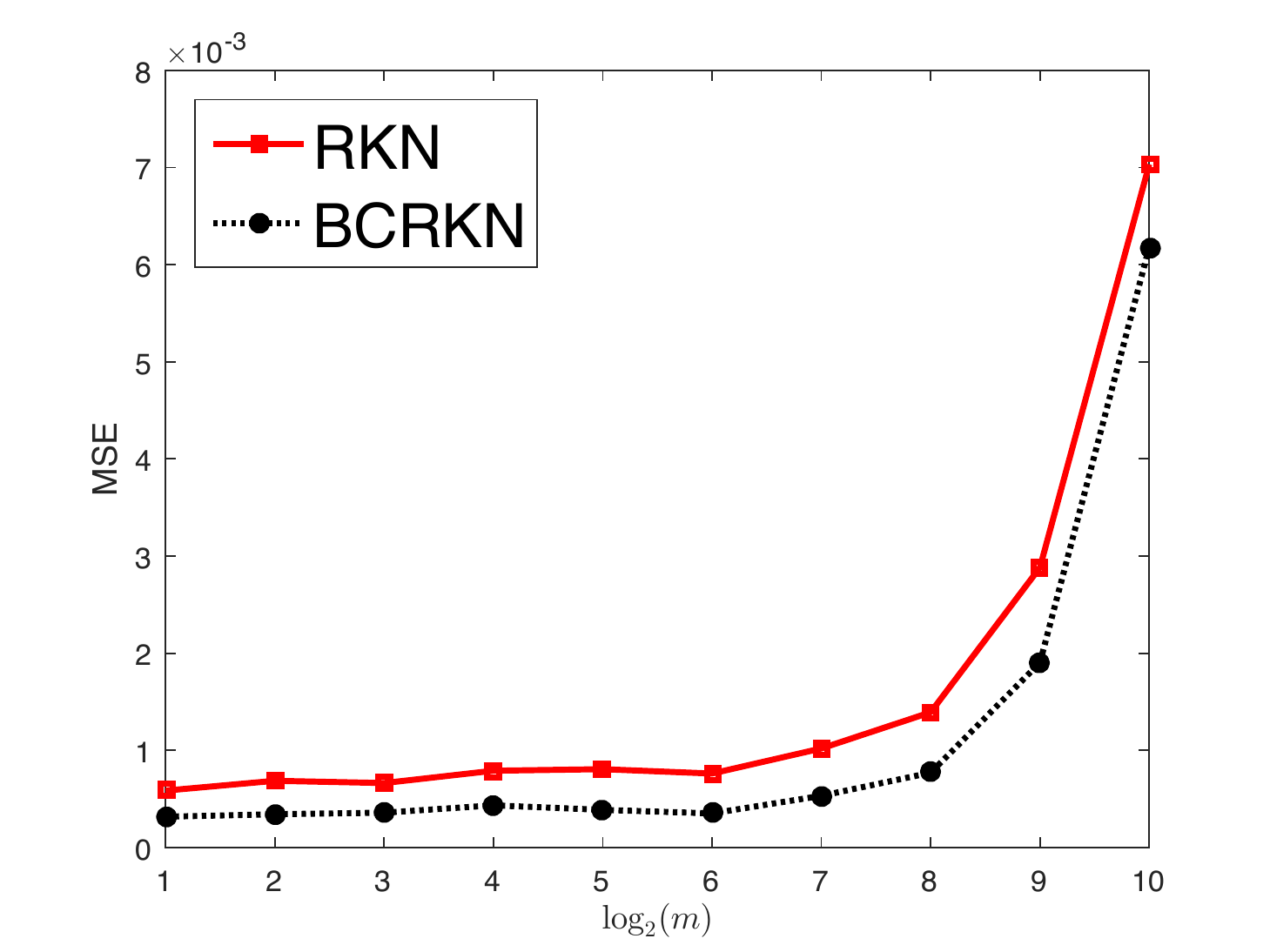}
\hfill
\includegraphics[width=0.49\textwidth]{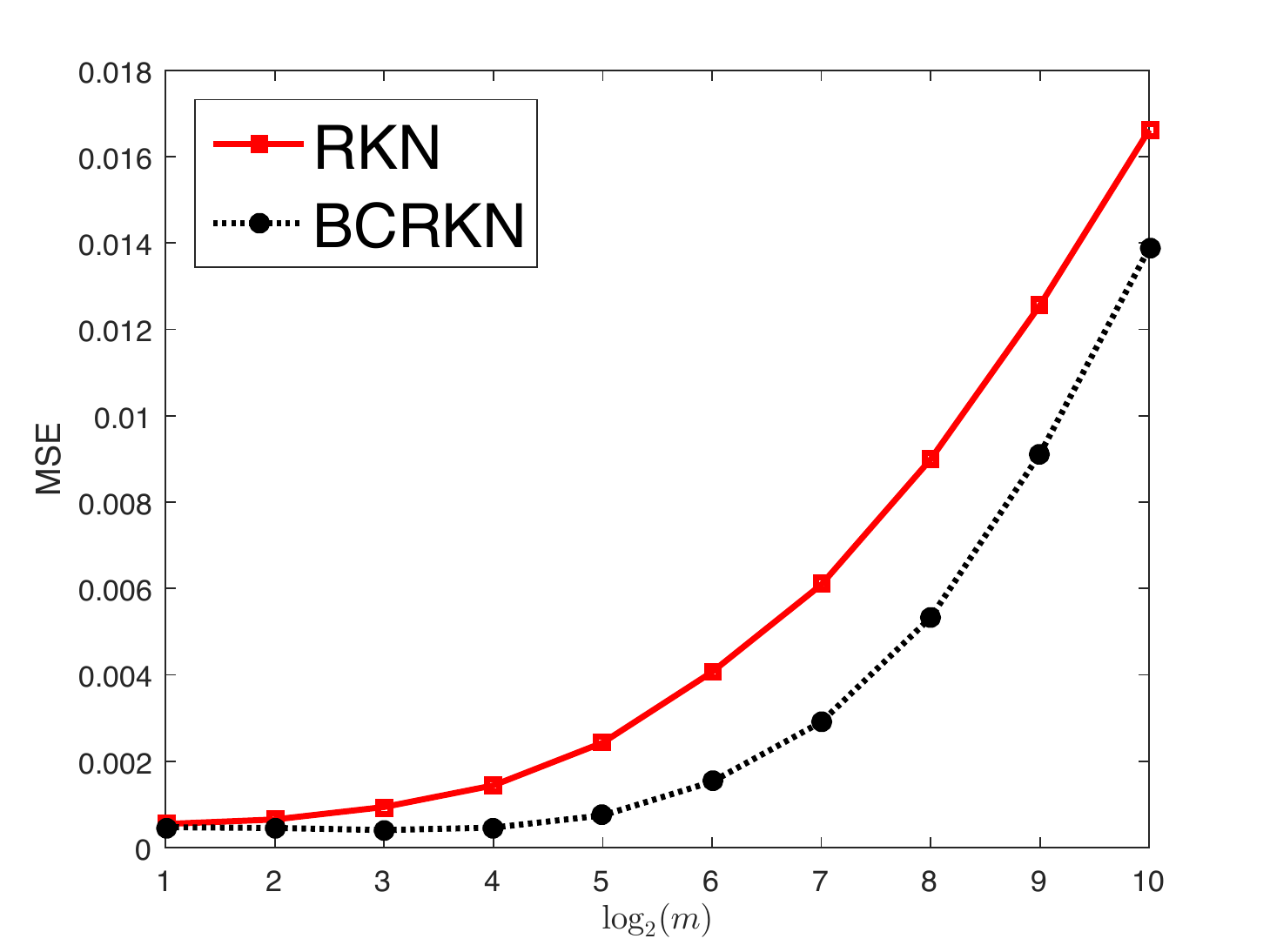}\\
\null\hfill (a) \hfill \hfill (b) \hfill \null \\
\caption{MSE of distributed RKN and distributed BCRKN. (a) $\lambda=|D|^{-2/3}$ is used.
(b)  $\lambda$ is  first tuned locally and the underregularized.
\label{figzh}}
\end{figure}

Finally, note that the upper bounds on the number of local processors are constrained
by the analysis techniques, not necessarily reflect the true limit
on the number of local processors allowed in practice.
Also, since underregularization is necessary in distributed regression but
globally tuning the optimal parameter is impractical,
the impact of parameter selection strategy on the number of local processors
is unknown. Further investigation on these issues is of great interest in future research.


\section{Preliminary lemmas}\label{sec:pre}

\begin{lemma}\label{lem: concentration inequality}
Let ${D}$ be a sample drawn independently according to $\rho$ and $g$ be a measurable bounded function on $\mathcal{Z}$ and $\xi_g$ be a random variable with values on $\mathcal{H}_K$ given by $\xi_g(z)=g(z)K_x$ for $z=(x,y)\in \mathcal{Z}.$
For any $0<\delta<1,$ with confidence at least $1-\delta,$ there holds
$$\left\|(\lambda I+L_K)^{-\frac12}\left(\frac{1}{|{D}|}\sum_{z\in D} \xi_g(z)-\bE[\xi_g]\right)\right\|_K
\le \frac{\|g\|_\infty \log\frac{2}{\delta}}{\kappa}\B_{|D|,\lambda}$$
\end{lemma}

\begin{lemma}\label{lem: DeltaD bound}
Let  $D$ be a sample drawn independently according to $\rho.$ If $|y|\le M$ almost surely, then
with confidence at least $1-\delta,$ there holds
\[
      \left\|(\lambda I+L_K)^{-\frac12}\tfrac{1}{|D|} (S_D^* \by_D - L_{K, D} f^*) \right\|_K 
            \le 2M\mathcal B_{|D|,\lambda}\log\frac{2}{\delta}.
\]
\end{lemma}

\begin{lemma}\label{lem: operator difference}
Let $D$ be a sample drawn independently according to $\rho.$ If $|y|\le M$ almost surely, then
for any $0<\delta <1,$
with confidence at least $1-\delta,$ there holds
\begin{equation}\label{operator difference norm 2}
       \Xi_D:=  \|(\lambda I+L_K)^{-\frac12}(L_K-L_{K,D})\|\le
          \mathcal B_{|D|,\lambda}\log\frac{2}{\delta},
\end{equation}
\end{lemma}

If $A$ and $B$ are invertible operators on a Banach space,
then by the second order operator decomposition proposed in \citep{LGZ16}, we have
\begin{eqnarray}\label{Second order}
          A^{-1}-B^{-1}
          =
          B^{-1}(B-A)B^{-1}(B-A)A^{-1}+B^{-1}(B-A)B^{-1}.
\end{eqnarray}
This implies the following decomposition of the operator product 
\begin{equation}\label{operator product based on second order}
     BA^{-1}
   =(B-A)B^{-1}(B-A)A^{-1}+(B-A)B^{-1}+I.
\end{equation}
With $A=L_{K,D}+\lambda I$ and $B=L_{K}+\lambda I$ in (\ref{operator product based on second order}),
and applying Lemma \ref{lem: operator difference}, we have the following bound for $\|(L_K+\lambda I)(L_{K,D}+\lambda I)^{-1}\|$;
for the detailed proof see \citep{GLZ16}.

\begin{proposition}\label{prop: operator product bound}
For any $0<\delta<1,$ with confidence at least
$1-\delta,$ there holds
\[
         \Omega_D:= \|(L_K+\lambda I)(L_{K,D}+\lambda I)^{-1}\|\le
          \left(\frac{\mathcal B_{|D|,\lambda}\log\frac{2}{\delta}}
           {\sqrt{\lambda}}   +1\right)^2.
\]
Moreover, the confidence set is the same as that in Lemma \ref{lem: operator difference}.
\end{proposition}

\begin{lemma} \label{P2E}
Let $Q$ be positive random variable.
If there are constants $a>0, b>0, \tau>0$ such that for any $0<\delta\le 1$, with confident at least $1-\delta,$ there holds
$Q\le a (\log \frac b \delta)^\tau,$ then for any $s>0$ we have
$\bE[Q^s] \le a^s b \Gamma(\tau s +1).$
\end{lemma}

\begin{proof}
Note the condition implies that  for all $t>0$ there is
$$\Pr\left[Q^{\frac 1 \tau} >t\right] \le b \exp\left(-\frac {t}{a^{1/\tau}}\right).$$ So we have
\begin{eqnarray*}
\bE[Q^s]  & =  & \bE\left[\left(Q^{1/\tau}\right)^{\tau s} \right ]
 =  \tau s \dint_0^\infty t^{\tau s -1}  \Pr\left[Q^{\frac 1 \tau} >t\right]  dt \\
& \le & \tau s  b \dint_0^\infty t^{\tau s -1}  \exp \left(-\frac {t} {a^{1/\tau}}\right) dt  \\
&  = & b  \tau s  \Gamma (\tau s) a ^s = a^s b \Gamma (\tau s +1).
\end{eqnarray*}
This proves the lemma.
\end{proof}


\section{Error analysis of BCRKN in $\L2p$ when $r\ge \frac 12$}
\label{sec:bigr}

We will split proof of Theorem \ref{thm:error2} into three cases: $0<r<\frac 12$, $\frac 1 2 \le r\le \frac 32$, and  $\frac 32 \le r\le 2.$
In this section we prove it for the second and third cases while leave the first case to Section \ref{sec:semi}.
Denote $\Delta_D = \frac 1 {|D|} S_D^* ( \by_D - S_D f^*).$

\begin{proposition}
\label{prop: error decomposition2}
If $\frac12 \le r \le \frac32,$  we have
\begin{eqnarray*}
\|f_{D,\lambda}^\sharp-f^*\|_\L2p \le 2\Omega_D \left\|(\lambda I+L_K) ^{-\frac12}\Delta_D\right\|_K +\lambda^r (\Omega_D)^{r}  \|u^*\|_\L2p.
\end{eqnarray*}
\end{proposition}

\begin{proof}
By the triangle inequality, we have
\begin{equation}
\label{eq:errorsplit}
\left\|f_{D,\lambda}^\sharp-f^*\right\|_\L2p\le \left\|f_{D,\lambda}^\sharp-\bE^*[f_{D,\lambda}^\sharp]\right\|_\L2p
+\left\|\bE^*[f_{D,\lambda}^\sharp]-f^*\right\|_\L2p,
\end{equation}
where
$ \bE^*[f_{D,\lambda}^\sharp]=(2\lambda I+L_{K,D})(\lambda I+L_{K,D})^{-2}L_{K,D} f^* $
is the conditional expectation with respect to $\by_D$ given $\bx_D.$

For the first term $\left\|f_{D,\lambda}^\sharp-\bE^*[f_{D,\lambda}^\sharp]\right\|_\L2p$,
noting that $2\lambda I+L_{K,D}$ and $(\lambda I+L_{K,D})^{-1}$ commute,
we have
\begin{eqnarray}
&& \left\|f_{D,\lambda}^\sharp-\bE^*[f_{D,\lambda}^\sharp]\right\|_\L2p 
=  \left\|L_K ^{\frac12} (2\lambda I+L_{K,D})(\lambda I+L_{K,D})^{-2}\Delta_D\right\|_K  \nonumber \\
& \le  &  \left\| (\lambda I+L_K) ^{\frac12}(\lambda I+L_{K,D}) ^{-\frac12} \right\|
\left\|(\lambda I+L_{K,D}) ^{-\frac12} (2\lambda I + L_{K,D})(\lambda I+L_{K,D}) ^{-\frac12}\right\|  \nonumber \\
& & \quad \times
     \left\| (\lambda I+L_{K,D}) ^{-\frac12}  (\lambda I+L_{K})^{\frac 12} \right\|
     \left\|(\lambda I+L_{K}) ^{-\frac12} \Delta \right\|_K \nonumber \\
& \le & 2 \Omega_D  
\left\|(\lambda I+L_K) ^{-\frac12}\Delta_D\right\|_K,
\label{eq:e1est}
\end{eqnarray}
here we have used the fact \citep{Blanchard2010} that
$$\|A^s B^s\|\le \|AB\|^s,\qquad 0\le s\le1,$$
for positive operators $A$ and $B$ on Hilbert spaces.

For  the second term,  we have  
\begin{eqnarray*}
\bE^*[f_{D,\lambda}^\sharp]-f^*=[(2\lambda I+L_{K,D})(\lambda I+L_{K,D})^{-2}L_{K,D}-I]f^*
=\lambda^2(\lambda I+L_{K,D})^{-2} f^*.
\end{eqnarray*}
By the regularity condition \eqref{regularitycondition},
\begin{eqnarray}
 \left\|\bE^*[f_{D,\lambda}^\sharp]-f^*\right\|_\L2p & =
&\lambda^2\left\|(\lambda I+L_{K,D})^{-2} L^r u^* \right\|_\L2p \nonumber\\
&\le &\lambda^2\left\|(\lambda I+L_K)^{\frac12}(\lambda I+L_{K,D})^{-2}L_K^{r-\frac12} L^{\frac12}u^*\right\|_K
\nonumber \\
&\le &\lambda^2 \left\|(\lambda I+L_{K})^{\frac12}(\lambda I+L_{K,D})^{-\frac12}\right\| \left\|(\lambda I+L_{K,D})^{-\frac32}L_K^{r-\frac12}\right\|  \left\|L^{\frac 12}u^*\right\|_K \nonumber \\
&\le &\lambda^2 (\Omega_D)^{\frac12} \left\|(\lambda I+L_{K,D})^{-\frac32}L_K^{r-\frac12}\right\|  \|u^*\|_\L2p. \label{eq:secondpart}
\end{eqnarray}

Since $\frac12\le r\le \frac32,$ we have
\begin{eqnarray*}
 &&\left\|(\lambda I+L_{K,D})^{-\frac32}L_K^{r-\frac12}\right\|\\
&=&\left\| (\lambda I+L_{K,D})^{r-2} (\lambda I+L_{K,D})^{-r+\frac12} (\lambda I+L_K)^{r-\frac12} (\lambda I+L_K)^{-r+\frac12} L_K^{r-\frac12}\right\|\\
&\le&\left\|(\lambda I+L_{K,D})^{r-2}\right\| \left\|(\lambda I+L_{K,D})^{-r+\frac12} (\lambda I+L_K)^{r-\frac12}\right\| \left\|(\lambda I+L_K)^{-r+\frac12} L_K^{r-\frac12}\right\|\\
&\le& \lambda^{r-2} (\Omega_D)^{r-\frac12}.
\end{eqnarray*}
Therefore,
\begin{equation}\label{E* case1}
\left\|\bE^*[f_{D,\lambda}^\sharp]-f^*\right\|_\L2p\le \lambda^r (\Omega_D)^{r}  \|u^*\|_\L2p.
\end{equation}
Then the conclusion follows by combining \eqref{eq:e1est} and \eqref{E* case1}.
\end{proof}

\begin{proposition}
\label{prop: error decomposition2-case3}
If  $\frac32\le r <2,$ we have
\begin{eqnarray*}
\|f_{D,\lambda}^\sharp-f^*\|_\L2p \le 2\Omega_D \left\|(\lambda I+L_K) ^{-\frac12}\Delta_D\right\|_K + \lambda \Xi_D
(\Omega_D)^{\frac32} \kappa^{2r-3} \|u^*\|_\L2p+\lambda^r \Omega_D \|u^*\|_\L2p.
\end{eqnarray*}
\end{proposition}

\begin{proof}
The proof is similar to Proposition \ref{prop: error decomposition2}.
First, $\|f_{D,\lambda}^\sharp-f^*\|_\L2p$ can be divided into two terms by \eqref{eq:errorsplit}.
The first term has been estimated in Proposition \ref{prop: error decomposition2} as \eqref{eq:e1est}.
We now focus on the second term. To this end, by \eqref{eq:secondpart},
we  only need to estimate $\left\|(\lambda I+L_{K,D})^{-\frac32}L_K^{r-\frac12}\right\|$.
When $\frac32\le r <2,$ we have
\begin{eqnarray*}
&&(\lambda I+L_{K,D})^{-\frac32}L_K^{r-\frac12}\\
&= & (\lambda I+L_{K,D})^{-\frac12}\left[(\lambda I+L_{K,D})^{-1}- (\lambda I+L_{K})^{-1}\right]L_K^{r-\frac12} \\
& &  +(\lambda I+L_{K,D})^{-\frac12} (\lambda I+L_{K})^{-1}L_K^{r-\frac12} \\
& = & (\lambda I+L_{K,D})^{-\frac12} (\lambda I+L_{K})^{-1} (L_K-L_{K,D})(\lambda I+L_{K,D})^{-1} (\lambda I+L_{K})(\lambda I+L_{K})^{-1} L_K^{r-\frac12}\\
&&\quad+(\lambda I+L_{K,D})^{-\frac12} (\lambda I+L_{K})^{\frac12} (\lambda I+L_{K})^{-\frac32}L_K^{r-\frac12}.
\end{eqnarray*}
By the bounds $\left\|(\lambda I+L_{K,D})^{-\frac12}\right\|\le \frac{1}{\sqrt{\lambda}},$ $\left\|(\lambda I+L_{K})^{-\frac12}\right\|\le \frac{1}{\sqrt{\lambda}},$ and $\|L_K\|\le \kappa^2,$ we have,
\begin{eqnarray*}
&&\left\|(\lambda I+L_{K,D})^{-\frac32}L_K^{r-\frac12}\right\|\\
&\le & \frac 1 \lambda  \left\|(\lambda I+L_{K})^{-\frac12} (L_K-L_{K,D})\right\|
\left\|(\lambda I+L_{K,D})^{-1} (\lambda I+L_{K})\right\|
\left\|(\lambda I+L_{K})^{-1} L_K^{r-\frac12}\right\|\\
&& \quad+\left\|(\lambda I+L_{K,D})^{-\frac12}(\lambda I+L_{K})^{\frac12} \right\| \left\|(\lambda I+L_{K})^{-\frac32}L_K^{r-\frac12}\right\|\\
&\le & \lambda^{-1} \Omega_D\left\|(\lambda I+L_{K})^{-\frac12} (L_K-L_{K,D})\right\| \kappa^{2r-3}+\lambda^{r-2}(\Omega_D)^{\frac12}.
\end{eqnarray*}
Therefore, putting the above bound back into \eqref{eq:secondpart} yields
\begin{eqnarray}\label{E* case2}
\left\|\bE^*[f_{D,\lambda}^\sharp]-f^*\right\|_\L2p\le    \lambda \left\|(\lambda I+L_{K})^{-\frac12} (L_K-L_{K,D})\right\| (\Omega_{D})^{\frac32} \kappa^{2r-3} \|u^*\|_\L2p+\lambda^r \Omega_{D} \|u^*\|_\L2p.
\end{eqnarray}
Now the conclusion follows by plugging \eqref{eq:e1est} and \eqref{E* case2} into \eqref{eq:errorsplit}.
\end{proof}

 Now we are ready to prove Theorem \ref{thm:error2}  and Corollary \ref{rate} for $r\ge \frac 12$.

 \bigskip

\noindent{\bf Proof of Theorem \ref{thm:error2}: Case $\frac 12\le r\le 2$.}\
By Lemma \ref{lem: DeltaD bound}, we have with confidence at least $1-\frac{\delta}{2},$
\begin{equation}\label{eq:1-1}
\left\|(\lambda I+L_K) ^{-\frac12}\Delta_D\right\|_K\le 2M\mathcal{B}_{|D|,\lambda}\log\frac4{\delta}.
\end{equation}
By Lemma \ref{lem: operator difference} and Proposition \ref{prop: operator product bound}, we obtain that,
with confidence at least $1-\frac \delta 2$,
\begin{equation}\label{eq:1-2}
   \Xi_D= \|(\lambda I+L_K)^{-\frac12}(L_K-L_{K,D})\| \le       \mathcal B_{|D|,\lambda}\log\frac{4}{\delta}
          \end{equation}
and
\begin{equation}\label{eq:1-3}
\Omega_D= \|(L_K+\lambda I)(L_{K,D}+\lambda I)^{-1}\|\le
          \left(\frac{\mathcal B_{|D|,\lambda}\log\frac{4}{\delta}}
           {\sqrt{\lambda}}   +1\right)^2
\end{equation}
hold simultaneously.

When $\frac{1}{2}\le r \le \frac{3}{2},$  we apply \eqref{eq:1-1} and \eqref{eq:1-3} to Proposition \ref{prop: error decomposition2}
and obtain
\begin{eqnarray*}
\|f_{D,\lambda}^\sharp-f^*\|_\L2p&\le &4M\left(\frac{\mathcal B_{|D|,\lambda}\log\frac{4}{\delta}}
           {\sqrt{\lambda}}   +1\right)^2\mathcal B_{|D|,\lambda}\log\frac{4}{\delta} + \lambda^r \left(\frac{\mathcal B_{|D|,\lambda}\log\frac{4}{\delta}}
           {\sqrt{\lambda}}   +1\right)^{2r}  \|u^*\|_\L2p\\
           &\le & (4M+\|u^*\|_\L2p)\left(\frac{\mathcal B_{|D|,\lambda}}
           {\sqrt{\lambda}}   +1\right)^{3}(\mathcal B_{|D|,\lambda}+\lambda^r)\left(\log\frac{4}{\delta}\right)^4
\end{eqnarray*}

When $\frac32 \le r \le 2,$ we apply \eqref{eq:1-1}, \eqref{eq:1-2} and \eqref{eq:1-3} to Proposition \ref{prop: error decomposition2-case3}
 and obtain
 \begin{eqnarray*}
\|f_{D,\lambda}^\sharp-f^*\|_\L2p&\le &4M\left(\frac{\mathcal B_{|D|,\lambda}\log\frac{4}{\delta}}
           {\sqrt {\lambda}}  +1\right)^2 \mathcal B_{|D|,\lambda}\log\frac{4}{\delta}\\
           &&+ \lambda \mathcal B_{|D|,\lambda}\log\frac{4}{\delta} \left(\frac{\mathcal B_{|D|,\lambda}\log\frac{4}{\delta}}
           {\sqrt{\lambda}}   +1\right)^3 \kappa^{2r-3} \|u^*\|_\L2p  \\
           && +\lambda^r \left(\frac{\mathcal B_{|D|,\lambda}\log\frac{4}{\delta}}
           {\sqrt{\lambda}}   +1\right)^2 \|u^*\|_\L2p\\
           &\le & (4M+2\kappa^{2r-3} \|u^*\|_\L2p)
           \left(\frac{\mathcal B_{|D|,\lambda}}
           {\sqrt{\lambda}}   +1\right)^3 (\mathcal B_{|D|,\lambda}+\lambda^r) \left(\log\frac{4}{\delta}\right)^4
\end{eqnarray*}

So \eqref{eq:boundP} are proved for all $\frac 1 2 \le r \le 2.$
Applying Lemma \ref{P2E} with $b=4$, $\tau = 4$, and $s=2$,
we have \eqref{eq:boundE} follows with
$C' = 4\Gamma(9)C^2.$
\hfill\BlackBox\\[2mm]

\noindent{\bf Proof of Corollary \ref{rate} (ii).}\
With  $\frac 12 \le r\le 2$ and the choice of $\lambda=|D|^{-\frac{1}{2r+\beta}},$ we have
\begin{equation}\label{Brate}
B_{|D|,\lambda}\le  \frac{2\kappa}{\sqrt{|D|}}  \left\{ \frac {\kappa |D|^{\frac{1}{2(2r+\beta)}} } {\sqrt{|D| }}
          +\sqrt{C_0}|D|^{\frac{\beta}{2(2r+\beta)}}\right\}
          \le 2\kappa\left(\kappa
          +\sqrt{C_0}\right)|D|^{-\frac{r}{2r+\beta}}.
\end{equation}
and
\begin{equation}\label{Bdlrate}
\frac{B_{|D|,\lambda}}{\sqrt \lambda} + 1 \le  2\kappa(\kappa
          +\sqrt{C_0})|D|^{-\frac{r}{2r+\beta}} |D|^{\frac 1{2r+\beta}} +1
          \le  2\kappa\left(\kappa  +\sqrt{C_0}\right) +1
 \end{equation}
Then the conclusions follow from Theorem \ref{thm:error2}.
\hfill\BlackBox\\[2cm]


\section{Error analysis in $\mathcal{H}_K$}
\label{sec:errHK}

In this section, we drive the error bound for $\|f_{D,\lambda}^\sharp-f^*\|_K$ and
prove the convergence of BCRKN in $\H_K.$
It is similar to the error analysis in $\L2p.$

\begin{proposition}
If $r\in[\frac12,\frac32],$ we have
\begin{eqnarray*}
\|f_{D,\lambda}^\sharp-f^*\|_K \le  2\lambda^{-\frac12} (\Omega_D)^{\frac12}
\left\|(\lambda I+L_{K})^{-\frac12}\Delta_D\right\|_K +\lambda^{r-\frac12} (\Omega_D)^{r-\frac12} \| u^*\|_\L2p.
\end{eqnarray*}
If  $r\in(\frac32,2],$ we have
\begin{eqnarray*}
\|f_{D,\lambda}^\sharp-f^*\|_K \le 2\lambda^{-\frac12} (\Omega_D)^{\frac12} \left\|(\lambda I+L_{K})^{-\frac12}\Delta_D\right\|_K
+ \lambda^{\frac12}\Xi_D \Omega_{D} \kappa^{2r-3} \|u^*\|_\L2p
+\lambda^{r-\frac12} \Omega_{D} \|u^*\|_\L2p
\end{eqnarray*}
\end{proposition}

\begin{proof}
By the triangle inequality in $\H_K$, we have
\begin{eqnarray*}
\|f_{D,\lambda}^\sharp-f^*\|_K\le \|f_{D,\lambda}^\sharp-\bE^*[f_{D,\lambda}^\sharp]\|_K+\|\bE^*[f_{D,\lambda}^\sharp]-f^*\|_K.
\end{eqnarray*}

To estimate the first term, we see that
$$f_{D,\lambda}^\sharp-\bE^*[f_{D,\lambda}^\sharp]=
(2\lambda I+L_{K,D})(\lambda I+L_{K,D})^{-2} \Delta.$$
Then 
\begin{eqnarray}\label{first part in HK}
&&\left\|f_{D,\lambda}^\sharp-\bE^*[f_{D,\lambda}^\sharp]\right\|_K=\|(2\lambda I+L_{K,D})(\lambda I+L_{K,D})^{-2}\Delta_D\|_K \nonumber \\
&\le &\left\|(2\lambda I+L_{K,D})(\lambda I+L_{K,D})^{-\frac32}\right\|
\left\|(\lambda I+L_{K,D})^{-\frac12}(\lambda I+L_{K})^{\frac12}\right\|
\left\|(\lambda I+L_{K})^{-\frac12}\Delta_D\right\|_K \nonumber\\
&\le & 2\lambda^{-\frac12} (\Omega_D)^{\frac12} \left\|(\lambda I+L_{K})^{-\frac12}\Delta_D\right\|_K.\nonumber
\end{eqnarray}

For the second term,  we have
\begin{eqnarray*}
\left\|\bE^*[f_{D,\lambda}^\sharp]-f^*\right\|_K
=\left\|\lambda^2(\lambda I+L_{K,D})^{-2} L^r u^*\right\|_K\le \lambda^{2}\|(\lambda I+L_{K,D})^{-2} L_K^{r-\frac12}\| \| u^*\|_\L2p.
\end{eqnarray*}
Following the same idea as in the proof of Proposition \ref{prop: error decomposition2}, we obtain
for $r\in[\frac12,\frac32],$
\begin{eqnarray*}
\left\|\bE^*[f_{D,\lambda}^\sharp]-f^*\right\|_K\le
\lambda^{r-\frac12} (\Omega_D)^{r-\frac12} \| u^*\|_\L2p
\end{eqnarray*}
and following the ideas in the proof of Proposition \ref{prop: error decomposition2-case3}, we obtain for $r\in(\frac32,2],$
\begin{eqnarray*}
\left\|\bE^*[f_{D,\lambda}^\sharp]-f^*\right\|_K \le
\lambda^{\frac12}\Xi_D\Omega_{D} \kappa^{2r-3} \|u^*\|_\L2p
+\lambda^{r-\frac12} \Omega_{D} \|u^*\|_\L2p.
\end{eqnarray*}

The desired error bounds now follow by combining the estimates for both terms.
\end{proof}

\noindent{\bf Proof of Theorem \ref{thm:errorK}.}\
Note that  \eqref{eq:1-1}, \eqref{eq:1-2} and \eqref{eq:1-3} hold simultaneously with probability
at least $1-\delta.$ Therefore, when $\frac 12 \le r\le \frac 32,$ we have with confidence at least $1-\delta$
\begin{eqnarray*}
\|f_{D,\lambda} ^\sharp - f^* \|_K
&\le & 2 \lambda^{-\frac12}\left(\frac{\mathcal B_{|D|,\lambda}\log\frac{4}{\delta}}
           {\sqrt{\lambda}}   +1\right) 2M\mathcal B_{|D|,\lambda}\log\frac{4}{\delta}  \\
     & & \quad +\lambda^{r-\frac12} \left(\frac{\mathcal B_{|D|,\lambda}\log\frac{4}{\delta}}
           {\sqrt{\lambda}}   +1\right)^{2r-1} \| u^*\|_\L2p\\
 &\le & (4M+\| u^*\|_\L2p)\left(\frac{\mathcal B_{|D|,\lambda}}
           {\sqrt{\lambda}}   +1\right)^2 (\lambda^{-\frac12} \mathcal B_{|D|,\lambda}+\lambda^{r-\frac12})\left(\log\frac{4}{\delta}\right)^3
\end{eqnarray*}
and, when $\frac 32 \le r\le 2,$
\begin{eqnarray*}
\|f_{D,\lambda} ^\sharp - f^* \|_K
&\le & 2\lambda^{-\frac12} \left(\frac{\mathcal B_{|D|,\lambda}\log\frac{4}{\delta}}
           {\sqrt{\lambda}}   +1\right) 2M\mathcal B_{|D|,\lambda}\log\frac{4}{\delta}  \\
           & & \quad + \lambda^{\frac12} \mathcal B_{|D|,\lambda}\log\frac{4}{\delta}  \left(\frac{\mathcal B_{|D|,\lambda}\log\frac{4}{\delta}}
           {\sqrt{\lambda}}   +1\right)^2 \kappa^{2r-3} \|u^*\|_\L2p\\
           &&\quad +\lambda^{r-\frac12} \left(\frac{\mathcal B_{|D|,\lambda}\log\frac{4}{\delta}}
           {\sqrt{\lambda}}   +1\right)^2 \|u^*\|_\L2p\\
&\le & \left(4M+2\kappa^{2r-3} \|u^*\|_\L2p\right)\left(\frac{\mathcal B_{|D|,\lambda}}
           {\sqrt{\lambda}}   +1\right)^2(\lambda^{-\frac12}\mathcal B_{|D|,\lambda}+\lambda^{r-\frac12})\left(\log\frac{4}{\delta}\right)^3.
\end{eqnarray*}
This proves the error bound \eqref{boundK}. Then \eqref{rateK} follows from estimates \eqref{Brate} and \eqref{Bdlrate},  and
\eqref{rateKE} follows by applying Lemma \ref{P2E}.
\hfill\BlackBox\\[2mm]

\section{Improve the error analysis by unlabelled data}
\label{sec:semi}

The error analysis for the semi-supervised approaches are more involved.
To the best of our knowledge, this is the first time we obtained the optimal learning rates in this case.
Before we move on, notice that Theorem \ref{thm:error2} with $0<r<\frac 12$ is a special case of
Theorem \ref{thm:error1} with $D'=D$ when there is no unlabeled data.
So upon finishing Theorem \ref{thm:error1},  we also obtain Theorem \ref{thm:error2} with $0<r<\frac 12$.

We need to introduce an intermediate function.
Recall $L$ is a compact operator on $\L2p.$  Let $\{\tau_i\}_{i=1}^\infty$ and $\{\psi_i\}_{i=1}^\infty$ be the
eigenvalues and eigenfunctions of $L$. Then $\{\psi_i\}_{i=1}^\infty$ form an orthonormal basis of $\L2p.$
Let $P_\lambda$  be the projection operator on $\L2p$ that projects each $f\in\L2p$ onto
the subspace spanned by $\{\psi_i: \tau_i\ge \lambda\},$ i.e.
$$P_\lambda f =  \sum_{\{i:\tau_i\ge \lambda\}}  \langle \psi_i, f\rangle_\L2p \psi_i, \qquad \forall\  f\in\L2p.$$
By the isomorphism property \eqref{eq:iso} of $L^{\frac 12},$
$\{\phi_i=\sqrt{\tau_i} \psi_i:\sigma_i>0\}$ form an orthonormal basis of $\H_K.$
Since $\{i:\tau_i\ge \lambda\}$ is a finite set, it is obvious
$P_\lambda f \in \H_K$ for all $f\in\L2p.$ Define  $f_\lambda^{tr}=P_\lambda f^*.$
We can bound $\|f_{D',\lambda}^\sharp-f^*\|_\L2p$ as follows.

\begin{proposition}\label{prop: error decomposition1}
We have
\begin{equation*}
\|f_{D',\lambda}^\sharp-f^*\|_\L2p\le I_1+I_2+I_3,
\end{equation*}
where
\begin{eqnarray*}
I_1&=&\left\|(\lambda I +L_K)^{\frac12} (2\lambda I+L_{K,D'})(\lambda I+L_{K,D'})^{-2}
\left(\tfrac 1 {|D'|} S_{D'}^* \by_{D'} - L_{K, D'}  f_\lambda^{tr}\right) \right\|_K\\
I_2&=&\left\|\lambda^2(\lambda I +L_K)^{\frac12}(\lambda I+L_{K,D'})^{-2}f_\lambda^{tr}\right\|_K,\\
I_3&=& \|f_\lambda^{tr}-f^*\|_\L2p.
\end{eqnarray*}
\end{proposition}

\begin{proof}
Note that
\begin{equation}\label{eq:fd2term}
\|f_{D',\lambda}^\sharp-f^*\|_\L2p\le \|f_{D',\lambda}^\sharp-f_\lambda^{tr}\|_\L2p+\|f_\lambda^{tr}-f^*\|_\L2p.
\end{equation}
Since  $f_{D',\lambda}^\sharp-f_\lambda^{tr}\in\H_K,$ by the isometry property \eqref{eq:iso} of $L^{\frac 12} = L_K^{\frac 12}$, we have
\begin{equation}
\label{eq:u1}
 \|f_{D',\lambda}^\sharp-f_\lambda^{tr}\|_\L2p = \|L_K^{\frac12}(f_{D',\lambda}^\sharp-f_\lambda^{tr})\|_K
 \le \|(\lambda I +L_K)^{\frac12}(f_{D',\lambda}^\sharp-f_\lambda^{tr})\|_K.
\end{equation}
Recall that
 $$f_{D',\lambda}^{\sharp}=f_{D',\lambda}+\lambda(\lambda I+L_{K,D'})^{-1}f_{D',\lambda}
=(2\lambda I+L_{K,D'})(\lambda I+L_{K,D'})^{-2}\frac{1}{|D'|} S_{D'}^* \by_{D'} .$$
It is easy to check that
\begin{eqnarray*}f_{D',\lambda}^\sharp-f_\lambda^{tr} & =  &
(2\lambda I+L_{K,D'})(\lambda I+L_{K,D'})^{-2}  \left(\frac 1 {|D'|} S_{D'}^* \by_{D'} - L_{K,D'} f_\lambda^{tr}\right)  \\
& & \quad - \lambda^2 (\lambda I+L_{K,D'})^{-2}f_\lambda^{tr}.
\end{eqnarray*}
Putting this in \eqref{eq:u1} we have
$ \|f_{D',\lambda}^\sharp-f_\lambda^{tr}\|_\L2p$
bounded by $I_1+I_2$.
Together with \eqref{eq:fd2term}, we obtain  the desired conclusion.
\end{proof}

Next we estimate the three terms respectively.
The third term $I_3$ can be easily bounded by the following lemma, which has been proved in \citep{Caponnetto2006}.
\begin{lemma}\label{lem: flambda bound}
We have
$\|f_\lambda^{tr}-f^*\|_\L2p\le \lambda^r \|u^*\|_\L2p$
and
$\|f_\lambda^{tr}\|_K\le \lambda^{-\frac{1}{2}+r}\|u^*\|_\L2p.$
\end{lemma}

For the first term $I_1$, we have the following bound.
\begin{proposition}\label{prop: I1bound}
For any $\delta\in(0,1),$ with confidence at least $1-\delta,$ there holds
$$I_1 \le \left(\frac{2M}{\kappa}+2\|u^*\|_\L2p\right)\left(\frac{\B_{|D'|,\lambda}}
           {\sqrt{\lambda}}   +1\right)^3\left( \B_{|{D}|,\lambda}+  \lambda^r\right)\left(\log\frac{4}{\delta}\right)^3. $$
\end{proposition}

\begin{proof}
Since $ 2\lambda I+L_{K,D'}$ and $(\lambda I+L_{K,D'})^{-1}$ commute, we have
\begin{eqnarray*}
 I_1&= &\left\|(\lambda I +L_K)^{\frac12} (2\lambda I+L_{K,D'})(\lambda I+L_{K,D'})^{-2}
 \left(\frac 1 {|D'|} S_{D'}^* \by_{D'} - L_{K,D'} f_\lambda^{tr}\right)\right\|_K \\
&\le & \left\|(\lambda I +L_K)^{\frac12} (\lambda I +L_{K,{D'}})^{-\frac12}\right\| \left\|(2\lambda I+L_{K,D'})(\lambda I+L_{K,D'})^{-1}\right\| \\
& &\quad \times  \left\|(\lambda I +L_{K,{D'}})^{-\frac12} (\lambda I +L_K)^{\frac12}\right\| \left\|(\lambda I +L_K)^{-\frac12}
\left(\frac 1 {|D'|} S_{D'}^* \by_{D'} - L_{K,D'} f_\lambda^{tr}\right) \right\|_K \\
&\le & 2\Omega_{D'} \left\| (\lambda I +L_K)^{-\frac12}
\left(\frac 1 {|D'|} S_{D'}^* \by_{D'} - L_{K,D'} f_\lambda^{tr}\right) \right\|_K.
\end{eqnarray*}
Proposition \ref{prop: operator product bound} ensures that, with confidence at least $1-\frac \delta 2$,
\begin{equation}\label{estimation of Omega}
\Omega_{D'}\le \left(\frac{\B_{|D'|,\lambda}}
           {\sqrt{\lambda}}  +1 \right)^2 \left(\log\frac{4}{\delta}\right)^2. 
\end{equation}
Now it suffices to consider the term $ \left\| (\lambda I +L_K)^{-\frac12}
\left(\frac 1 {|D'|} S_{D'}^* \by_{D'} - L_{K,D'} f_\lambda^{tr}\right) \right\|_K.$
We further divide it into three parts as follows
\begin{eqnarray*}
\left\| (\lambda I +L_K)^{-\frac12}
\left(\frac 1 {|D'|} S_{D'}^* \by_{D'} - L_{K,D'} f_\lambda^{tr}\right) \right\|
& \le & \left\| (\lambda I +L_K)^{-\frac12}\left(\frac{1}{|D'|}S_{D'}^* \by _{D'} - L_K f^*\right) \right\|_K \\
&& \quad+ \left\| (\lambda I +L_K)^{-\frac12}(L_K f^*-L_K f_\lambda^{tr})\right\|_K \\
&&\quad+ \left\| (\lambda I +L_K)^{-\frac12}(L_K - L_{K,{D'}} )f_\lambda^{tr}\right\|_K.
\end{eqnarray*}
By the definition of $y_i'$, it is easy to check that
$\frac 1 {|D'|}  S_{D'}^* \by_{D'} = \frac 1 {|D|} S_D^* \by_D.$
Applying Lemma \ref{lem: concentration inequality} with  $\xi_g(z)=yK_x$, we obtain, with confidence at least $1-\frac \delta 2$,
$$\left\| (\lambda I +L_K)^{-\frac12}
\left(\frac 1 {|D'|} S_{D'}^* \by_{D'} - L_K f^*\right) \right\|_K
\le \frac{M \log\frac{4}{\delta}}{\kappa}\B_{|{D}|,\lambda}.$$
By Lemma \ref{lem: flambda bound}, we have
$$\left\| (\lambda I +L_K)^{-\frac12}(L_K f^*-L_K f_\lambda^{tr})\right\|_K \le\| f^*-f_\lambda^{tr}\|_\L2p \le \lambda^r\|u^*\|_\L2p.$$
For $\left\| (\lambda I +L_K)^{-\frac12}(L_K - L_{K,{D'}}) f_\lambda^{tr}\right\|_K,$
observe that
$$\left\| (\lambda I +L_K)^{-\frac12}(L_K - L_{K,{D'}}) f_\lambda^{tr}\right\|_K\le \left\| (\lambda I +L_K)^{-\frac12}(L_K - L_{K,{D'}})\right\|\| f_\lambda^{tr}\|_K.$$
By Lemma  \ref{lem: flambda bound}, we have $\|f_\lambda^{tr}\|_K\le \lambda^{-\frac{1}{2}+r}\|u^*\|_\L2p.$
By Lemma \ref{lem: operator difference}, we have with confidence at least $1-\frac \delta 2,$
$$\left\| (\lambda I +L_K)^{-\frac12}(L_K - L_{K,{D'}}) f_\lambda^{tr}\right\|_K \le
\lambda^{r}\frac{\B_{|D'|,\lambda}}{\sqrt{\lambda}}\|u^*\|_\L2p\log\frac{4}{\delta}$$
with the confidence set the same as that for \eqref{estimation of Omega}.
Combining the above estimations together yields
\begin{eqnarray}
&&\left\|(\lambda I +L_K)^{-\frac12}
\left( \frac 1 {|D'|} S_{D'}^* \by _{D'} -L_{K,{D'}} f_\lambda^{tr}\right)\right\|_K \nonumber \\[1em]
&\le &\left(\frac{M}{\kappa}+\|u^*\|_\L2p\right)\left(\frac{\B_{|D'|,\lambda}}{\sqrt{\lambda}}+1\right)
\left(\B_{|{D}|,\lambda}+\lambda^r\right)\log\frac{4}{\delta}.
\label{estimation of the second term}
\end{eqnarray}
Then our desired result follows by $\eqref{estimation of Omega}$ and $\eqref{estimation of the second term}.$
\end{proof}

\begin{proposition}\label{prop: I2bound}
For any $\delta\in(0,1),$ we have, with confidence at least $1-\frac \delta 2,$
$$I_2\le \left(\frac{\B_{|D'|,\lambda}}
           {\sqrt{\lambda}}   +1\right) \lambda^r\|u^*\|_\L2p \log\frac{4}{\delta}$$
with the confidence set the same as that for \eqref{estimation of Omega}.
\end{proposition}

\begin{proof} Now we are in a position to estimate the term $I_2,$ we decompose the term as
\begin{eqnarray*}
I_2&= &\|\lambda^2(\lambda I +L_K)^{\frac12}(\lambda I +L_{K,D'})^{-2} f_\lambda^{tr}\|_K\\
&= & \lambda^2\left\|(\lambda I +L_K)^{\frac12}(\lambda I +L_{K,{D'}})^{-\frac12}(\lambda I +L_{K,{D'}})^{-\frac32} f_\lambda^{tr} \right\|_K\\
&\le & \lambda^2 \|(\lambda I +L_K)^{\frac12}(\lambda I +L_{K,{D'}})^{-\frac12}\| \|(\lambda I +L_{K,D'})^{-\frac32}\| \|f_\lambda^{tr}\|_K\\
&\le & \lambda^r (\Omega_{D'})^{\frac12}\|u^*\|_\L2p,
\end{eqnarray*}
where we have used the bounds $\|(\lambda I +L_{K,D'})^{-\frac32}\|\le \lambda^{-\frac32}$ and $\|f_\lambda^{tr}\|_K\le \lambda^{r-\frac12}\|u^*\|_\L2p.$
By \eqref{estimation of Omega}, we obtain the desired bound and confidence set.
\end{proof}

Now we can prove Theorem \ref{thm:error1}.\\[2mm]

\noindent{\bf Proof of Theorem \ref{thm:error1}.}\
Plugging the bounds of $I_1$, $I_2,$ and $I_3$ into  Proposition \ref{prop: error decomposition1},
we obtain the error bound for $ \|f_{D',\lambda}^\sharp-f^*\|_\L2p$ in \eqref{eq:semibound}.

If $\lambda=|D|^{-\frac{1}{2r+\beta}}$ with $0<r\le\frac12$ and  $\mathcal{N}(\lambda)\le C_0\lambda^{-\beta},$
we have
\begin{eqnarray*}
B_{|D|,\lambda}\le  \frac{2\kappa}{\sqrt{|D|}}\left\{{\kappa} |D|^{-\frac{1}{2}}|D|^{\frac{1}{2(2r+\beta)}}
          +\sqrt{C_0}|D|^{\frac{\beta}{2(2r+\beta)}}\right\}
          \le 2\kappa(\kappa
          +\sqrt{C_0})|D|^{\frac{r}{2r+\beta}}.
\end{eqnarray*}
Under the condition $|D'|\ge |D|^{\frac{1+\beta}{2r+\beta}},$ we have
\begin{eqnarray*}
\frac{\B_{|D'|,\lambda}}
           {\sqrt{\lambda}}&&=\frac{2\kappa}{\sqrt{|D'|\lambda}}\left\{\frac{\kappa}{\sqrt{|D'|\lambda}}
          +\sqrt{\mathcal{N}(\lambda)}\right\}\\
          &&\le \frac{2\kappa}{\sqrt{|D'|}}|D|^{\frac{1}{2(2r+\beta)}}\left\{{\kappa} |D'|^{-\frac{1}{2}}|D|^{\frac{1}{2(2r+\beta)}}
          +\sqrt{C_0}|D|^{\frac{\beta}{2(2r+\beta)}}\right\}\\
          &&\le  \frac{2\kappa}{\sqrt{|D'|}}|D|^{\frac{1}{2(2r+\beta)}}({\kappa}
          +\sqrt{C_0})|D|^{\frac{\beta}{2(2r+\beta)}} \\
          && \le {2\kappa}({\kappa} +\sqrt{C_0})
\end{eqnarray*}
Applying these two estimates to \eqref{eq:semibound}, we have for any $\delta\in(0,1),$ with confidence at least $1-\delta,$
\begin{eqnarray*}
\|f_{D',\lambda}^\sharp-f^*\|_\L2p\le
C_1|D|^{-\frac{r}{2r+\beta}}\left(\log\frac{4}{\delta}\right)^3
\end{eqnarray*}
where $$C_1=\left(\frac{2M }{\kappa}+ 4\|u^*\|_\L2p\right)\left[{2\kappa}({\kappa}
          +\sqrt{C_0}) +1\right]^3\left[2\kappa(\kappa
          +\sqrt{C_0})+1\right].$$
This completes the proof of Theorem \ref{thm:error1}.
\hfill\BlackBox\\[2mm]

Note we also proved Theorem \ref{thm:error2} with $0<r<\frac 12$ because it is a special case of
Theorem \ref{thm:error1} with $D'=D$. So we are in position to  prove Corollary \ref{rate} (i). \\[2mm]

\noindent{\bf Proof of  Corollary \ref{rate} (i).}\
 When $0<r\le \frac 12$, take $\lambda = |D|^{-\frac{1}{1+\beta}}.$ Then
$$ B_{|D|,\lambda}\le  \frac{2\kappa}{\sqrt{|D|}}     \left\{    \frac{ \kappa |D|^{\frac{1}{2(1+\beta)}} } {\sqrt{|D|} }
          +\sqrt{C_0}|D|^{\frac{\beta}{2(1+\beta)}}\right\}
          \le 2\kappa(\kappa +\sqrt{C_0})|D|^{-\frac{1}{2(1+\beta)}}
            \le 2\kappa(\kappa +\sqrt{C_0})|D|^{-\frac{r}{1+\beta}}.
$$
and
 $$ \frac{B_{|D|,\lambda}}{\sqrt \lambda}  \le
 2\kappa(\kappa +\sqrt{C_0})|D|^{-\frac{1}{2(1+\beta)}} |D|^{\frac{1}{2(1+\beta)}}
   =      2\kappa(\kappa +\sqrt{C_0}).
$$
Plugging them into the estimation \eqref{eq:boundP}
we obtain the desired learning rate.
\hfill\BlackBox\\[2mm]

\section{Error analysis for distributed BCRKN}
\label{sec:distributed}

The following lemma is analogous to the \cite[Proposition 4]{GLZ16}.
\begin{lemma}
Let $\overline{f}_{D,\lambda}^\sharp$ be defined by \eqref{distributedlearningalgorithm}. We have
\begin{equation}\label{Prop on error decomp}
        \bE\left[\|\overline{f}_{D,\lambda}^\sharp-f^*\|_\L2p^2\right]
        \leq
         \sum_{j=1}^m\frac{|D_j|^2}{|D|^2} \bE\left[\|f_{D_j,\lambda}^\sharp-f^*\|_\L2p^2\right]
        +\sum_{j=1}^m\frac{|D_j|}{|D|}
        \left\| \bE[f_{D_j,\lambda}^\sharp]-f^*\right\|_\L2p^2.
\end{equation}
\end{lemma}


\noindent{\bf Proof of Theorem \ref{thm: distributed main result1}.}\
By Theorem \ref{thm:error2}, for each fixed   $j\in\{1,2,\ldots,m\}$,
\begin{eqnarray*}
     \bE\left[\|f_{D_j,\lambda}^\sharp-f^*\|^2_\L2p\right]\leq
    4\Gamma(9)C^2 \left(\frac{\mathcal B_{|D_j|,\lambda}}{\sqrt{\lambda}} +1\right)^6
     \left(\mathcal  B_{|D_j|,\lambda}+\lambda^r\right)^2.
\end{eqnarray*}
Then the first term on the right of  \eqref{Prop on error decomp} can be
estimated as
\begin{eqnarray}
   \sum_{j=1}^m\frac{|D_j|^2}{|D|^2}E\left[\|f_{D_j,\lambda}^\sharp-f^*\|^2_\L2p\right]
    \leq
    4\Gamma(9)C^2  \sum_{j=1}^m\frac{|D_j|^2}{|D|^2}
    \left(\frac{\mathcal B_{|D_j|,\lambda}}{\sqrt{\lambda}} +1\right)^6
     \left(\mathcal  B_{|D_j|,\lambda}+\lambda^r\right)^2.
      \label{Bound for Sample error for d smaller1.5}
\end{eqnarray}

We turn to estimate the second term on the right of \eqref{Prop on error decomp}.
For each fixed $j\in\{1,2,\ldots,m\}$, by  Jensen's inequality, we
have
\[
              \left\| \bE[{f}_{D_j,\lambda}^\sharp]-f^*\right\|_\L2p
           \leq
            \bE  \left[\|\bE^*[ {f}_{D_j,\lambda}^\sharp]-f^*\|_\L2p\right].
\]
We consider the second terms into two cases according to the range of $r$. We first consider the case when $\frac12\le r\le \frac32.$
The bound \eqref{E* case1} in the proof of Proposition
\ref{prop: error decomposition2} tells us that
\[
      \left\|\bE^*[f_{D_j,\lambda}^\sharp]-f^*\right\|_\L2p\leq
      \lambda^r \Omega_{D_j}^{r}  \|u^*\|_\L2p.
\]
It follows  that
\begin{equation}\label{D.11}
      \|
           \bE[{f}_{D_j,\lambda}^\sharp]-f^*\|_\L2p
            \leq \lambda^r   \|u^*\|_\L2p \bE\left[\Omega_{D_j}^r\right].
\end{equation}
Applying Proposition \ref{prop: operator product bound} to each fixed $j\in\{1,\dots,m\}$, with
confidence at least $1-\frac{\delta}{2}$, there holds
\[
            \Omega_{D_j}\leq \left(\frac{\mathcal B_{|D_j|,\lambda}\log\frac{4}{\delta}}
           {\sqrt{\lambda}}   +1\right)^2\le \left(\frac{\mathcal B_{|D_j|,\lambda}}
           {\sqrt{\lambda}}  +1 \right)^2\left(\log\frac{4}{\delta}\right)^2.
\]
By Lemma \ref{P2E}, this implies that for any $s>0$,
\begin{eqnarray}
        \bE\left[\Omega_{D_j}^s\right]  \le 2 \Gamma(2s+1)
        \left(\frac{\mathcal B_{|D_j|,\lambda}} {\sqrt{\lambda}}   +1\right)^{2s}.
\label{OmegeD Expectation}
\end{eqnarray}
Applying \eqref{OmegeD Expectation} with $s=r$ to  \eqref{D.11}  yields
\begin{equation}\label{Bound for approximation error small 1.5}
          \|\bE[{f}_{D_j,\lambda}^\sharp]-f^*\|_\L2p\leq
          2\Gamma(2r+1)\|u^*\|_\L2p\lambda^r \left(\frac{\mathcal B_{|D_j|,\lambda}}
           {\sqrt{\lambda}} +1\right)^{2r}.
\end{equation}
Combining \eqref{Prop on error decomp}, \eqref{Bound for Sample error for d smaller1.5} and
\eqref{Bound for approximation error small 1.5}, we have
\begin{eqnarray*}
         \bE[\|\overline{f}_{D,\lambda}^\sharp-f^*\|_\L2p^2]
        & \leq &
        4\Gamma(9)C^2
      \sum_{j=1}^m\frac{|D_j|^2}{|D|^2}\left(\frac{\mathcal B_{|D_j|,\lambda}}{\sqrt{\lambda}} +1\right)^6
      \left(\mathcal  B_{|D_j|,\lambda}+\lambda^r\right)^2\\
      && \quad +
       4   \Gamma^2(2r+1)\|u^*\|^2_\L2p \lambda^{2r}\sum_{j=1}^m\frac{|D_j|}{|D|}
          \left(\frac{\mathcal B_{|D_j|,\lambda}} {\sqrt{\lambda}}
          +1\right)^{4r}.
\end{eqnarray*}
This proves the desired bound for $\frac12\leq r\le \frac32$.

For $r\in(\frac32,2],$ by the bound \eqref{E* case2} in the proof of Proposition \ref{prop: error decomposition2-case3}, we have
\[
      \|\bE^*[f_{D_j,\lambda}^\sharp]-f^*\|_\L2p\leq
      \lambda \Xi_{D_j} \Omega_{D_j}^{\frac32} \kappa^{2r-3} \|u^*\|_\L2p+\lambda^r \Omega_{D_j} \|u^*\|_\L2p.
\]
So,
\begin{eqnarray*}
     \bE[\|\bE^*[f_{D_j,\lambda}^\sharp]-f^*\|_\L2p]
     &\le &      \lambda \bE\left[\Xi_{D_j} (\Omega_D)^{\frac32}\right]
     \kappa^{2r-3} \|u^*\|_\L2p+\lambda^r \bE[\Omega_{D_j}] \|u^*\|_\L2p\\
      &\le & \lambda \left(\bE\left[\Xi_{D_j}^2\right]\right)^{\frac12}
      \left( \bE\left[\Omega_{D_j}^{3}\right]\right)^\frac12 \kappa^{2r-3} \|u^*\|_\L2p+\lambda^r \bE[\Omega_{D_j}] \|u^*\|_\L2p.
\end{eqnarray*}
From \citep{LGZ16}, we have
\begin{eqnarray*}
\bE\left[\left\|(\lambda I+L_{K})^{-\frac12} (L_K-L_{K,D})\right\|^2\right]\le \frac{ \kappa^2\mathcal{N}(\lambda)}{|D|}
\end{eqnarray*}
and by \eqref{OmegeD Expectation} with $s=3,$ we have
\begin{eqnarray*}
\bE\left[\Omega_{D_j}^{3}\right]\le 2\Gamma(7)\left(\frac{\mathcal B_{|D_j|,\lambda}}
           {\sqrt{\lambda}}   +1\right)^{6}.
\end{eqnarray*}
Therefore,
\begin{eqnarray}
 \|\bE[f_{D_j,\lambda}^\sharp]-f^*\|_\L2p^2
     &\le &2\lambda^2 E\left[\Xi_{D_j}^2\right] E\left[\Omega_{D_j}^{3}\right] \kappa^{4r-6} \|u^*\|_\L2p^2+ 2\lambda^{2r} (\bE[\Omega_{D_j}])^2 \|u^*\|_\L2p^2 \nonumber\\
      &\le &2\lambda^2 \frac{ \kappa^2\mathcal{N}(\lambda)}{|D_j|} 2\Gamma(7)\left(\frac{\mathcal B_{|D_j|,\lambda}}
           {\sqrt{\lambda}}  +1\right)^{6}\kappa^{4r-6} \|u^*\|_\L2p^2 \nonumber \\
           & & +32\lambda^{2r}\|u^*\|_\L2p^2 \left(\frac{\mathcal B_{|D_j|,\lambda}}
           {\sqrt{\lambda}}   +1\right)^4 \nonumber\\
      &\le & \left(4\Gamma(7)\kappa^{4r-4}+32\right)\|u^*\|_\L2p^2 \left(\frac{\mathcal B_{|D_j|,\lambda}}
           {\sqrt{\lambda}}  +1\right)^6 \left(\frac{\lambda^2\mathcal{N}(\lambda)}{|D_j|}+\lambda^{2r}\right).
           \label{eq:fdjcase2}
\end{eqnarray}
Combining \eqref{Prop on error decomp}, \eqref{Bound for Sample error for d smaller1.5} and \eqref{eq:fdjcase2},
we have
\begin{eqnarray*}
         \bE[\|\overline{f}_{D,\lambda}^\sharp-f^*\|_\L2p^2]& \le &
        4\Gamma(9)C^2
      \sum_{j=1}^m\frac{|D_j|^2}{|D|^2}\left(\frac{\mathcal B_{|D_j|,\lambda}}{\sqrt{\lambda}}
            +1\right)^{6} \left(\mathcal B_{|D_j|,\lambda}+\lambda^r\right)^2\\
      &&  +
        \left(4\Gamma(7)\kappa^{4r-4}+32\right)\|u^*\|_\L2p^2\sum_{j=1}^m\frac{|D_j|}{|D|}
        \left(\frac{\mathcal B_{|D_j|,\lambda}}
           {\sqrt{\lambda}}   +1\right)^6\left(\frac{\lambda^2\mathcal{N}(\lambda)}{|D_j|}+\lambda^{2r}\right).
\end{eqnarray*}
This proves the conclusion for $\frac 32\le r \le 2.$
\hfill\BlackBox\\[2mm]

Next let us turn to the special case that the local machines are assigned the same number of samples, i.e., $|D_1|=|D_2|=\dots=|D_m|.$\\[2mm]

\noindent{\bf Proof of Theorem \ref{thm: distributed main result2}.}\
For $\frac12\leq r\leq2$, we choose $\lambda=|D|^{-\frac{1}{2r+\beta}}.$
Then by the capacity assumption \eqref{effecdim},
restriction on the number of local machines $m\leq |D|^{\min\{\frac{2r-1}{2r+\beta},\frac2{2r+\beta}\}},$
and the fact $|D_1|=|D_2|=\ldots=|D_m| = \frac {|D|}{m},$ we have
\[
         \frac{\mathcal N(\lambda)}{\lambda|D_j|}\leq C_0m
         N^{\frac{1-2r}{2r+\beta}}\leq C_0.
\]
It follows  that, for each $j=1, \ldots, m$,
\[
       \frac{\mathcal{B}_{|D_j|,\lambda}}{\sqrt{\lambda}}
       =\frac{2\kappa}{\sqrt{\lambda|D_j|}}\left\{\frac{\kappa}{\sqrt{|D_j|\lambda}}
       +\sqrt{\mathcal{N}(\lambda)}\right\}
       \leq
       2\kappa(\kappa+\sqrt{C_0})
\]
and
\[
         \frac{|D_j|}{|D|}\mathcal
      B^2_{|D_j|,\lambda}\leq 8\kappa^2\left(\frac{\kappa^2}{|D||D_j|\lambda}+\frac{\mathcal
      N(\lambda)}{|D|}\right)\leq 8\kappa^2(\kappa^2+{C_0}) |D|^{-\frac{2r}{2r+\beta}}.
\]
and $$\frac{\lambda^{2}\mathcal{N}(\lambda)}{{|D_j|}}\le\frac{C_0|D|^{-\frac{2}{2r+\beta}}|D|^{\frac{\beta}{2r+\beta}}m}{|D|}\le  C_0 |D|^{-\frac{2r}{2r+\beta}}.$$
Then by Theorem \ref{thm: distributed main result1},
\[
          \bE[\|\overline{f}_{D,\lambda}^\sharp-f^*\|_\L2p^2]
          \leq
          \bar{C}\left(2\kappa(\kappa+\sqrt{C_0})  +1\right)^6 \left(  16\kappa^2(\kappa^2+{C_0})+C_0+1 \right)
          |D|^{-\frac{2r}{2r+\beta}}.
\]
 This completes the proof of Theorem \ref{thm: distributed main result2}.
\hfill\BlackBox\\[2mm]

\acks{%
The work described in this paper is partially supported by the
National Natural Science Foundation of China
(Grants No.11401524, 11531013, 11571078, 11631015, 11671171).
Lei Shi is also supported by the Joint Research Fund by
National Natural Science Foundation of China and
Research Grants Council of Hong Kong
(Project No. 11461161006 and Project No. CityU 104012) and
Zhuo Xue Program of Fudan University.
Part of the work was carried out while Zheng-Chu Guo was visiting
Shanghai Key Laboratory for Contemporary Applied Mathematics.
All authors contributed equally to this paper and are listed alphabetically.
The corresponding author is Qiang Wu.
}

\bibliographystyle{abbrv}

\end{document}